\documentclass{article}
\usepackage{iclr2027_conference,times}

\usepackage{amsmath,amsfonts,bm}

\def\eqref#1{equation~\ref{#1}}
\def\1{\bm{1}}

\def\rva{{\mathbf{a}}}

\def\vb{{\bm{b}}}
\def\vc{{\bm{c}}}
\def\vd{{\bm{d}}}
\def\ve{{\bm{e}}}
\def\vf{{\bm{f}}}
\def\vg{{\bm{g}}}
\def\vh{{\bm{h}}}

\def\vr{{\bm{r}}}
\def\vs{{\bm{s}}}

\def\vv{{\bm{v}}}

\def\vx{{\bm{x}}}
\def\vy{{\bm{y}}}

\def\mA{{\bm{A}}}
\def\mB{{\bm{B}}}

\def\mF{{\bm{F}}}

\def\mR{{\bm{R}}}

\def\mW{{\bm{W}}}

\DeclareMathAlphabet{\mathsfit}{\encodingdefault}{\sfdefault}{m}{sl}
\SetMathAlphabet{\mathsfit}{bold}{\encodingdefault}{\sfdefault}{bx}{n}

\newcommand{\E}{\mathbb{E}}

\newcommand{\Cov}{\mathrm{Cov}}
\DeclareMathOperator*{\argmax}{arg\,max}
\DeclareMathOperator*{\argmin}{arg\,min}

\usepackage{amsmath,amssymb, amsthm}
\usepackage{graphicx}
\usepackage{hyperref}
\usepackage{url}
\usepackage{float}
\usepackage{booktabs}
\usepackage{caption}
\usepackage{subcaption}
\usepackage{color}
\usepackage{cleveref}
\crefformat{equation}{(#2#1#3)}
\Crefformat{equation}{(#2#1#3)}
\crefrangeformat{equation}{(#3#1#4) to (#5#2#6)}
\Crefrangeformat{equation}{(#3#1#4) to (#5#2#6)}
\crefmultiformat{equation}{(#2#1#3)}{ and (#2#1#3)}{, (#2#1#3)}{ and (#2#1#3)}
\Crefmultiformat{equation}{(#2#1#3)}{ and (#2#1#3)}{, (#2#1#3)}{ and (#2#1#3)}
\usepackage{algorithm}
\usepackage{algpseudocode}
\usepackage{titletoc}
\usepackage{tabularx}
\usepackage{xcolor}
\usepackage{colortbl}
\usepackage{tcolorbox}
\usepackage{multirow}
\usepackage{placeins}

\definecolor{headergray}{RGB}{85,85,85}
\definecolor{lightgray}{RGB}{245,245,245}

\newcolumntype{L}{>{\raggedright\arraybackslash}X}
\newcolumntype{C}[1]{>{\centering\arraybackslash}p{#1}}

\newtheorem{proposition}{Proposition}

\title{First-Order Steering: Translating Weight Adaptation into Activation Steering}

\author{Sri Pranav Kunda\textsuperscript{1,2} \quad
  Alexander Kurz\textsuperscript{2} \quad
  Tomáš Dominik\textsuperscript{2} \quad
  Uri Maoz\textsuperscript{2,3,4}\thanks{Corresponding author: \texttt{maoz@chapman.edu}} \\[0.5ex]
  \textnormal{\textsuperscript{1}University of California, Davis} \\
  \textnormal{\textsuperscript{2}Chapman University} \\
  \textnormal{\textsuperscript{3}University of California, Los Angeles} \\
  \textnormal{\textsuperscript{4}Caltech}
}

\iclrpreprintcopy

\begin{document}

\maketitle

\begin{abstract}
Activation steering exploits interpretable directions in the residual stream to enable inference-time manipulation of model behavior. Composing steering vectors to apply multiple target behaviors simultaneously is important in various fields—including AI alignment and safety—but remains a challenge for existing activation steering methods. In contrast, prior work in model merging shows that target behaviors represented by learned weight adaptations can be combined with high accuracy. A method that translates weight adaptations into activation steering vectors could therefore extend prior work in model merging to generate composable steering vectors that better enable simultaneous inference-time behavioral control. For this, we introduce \textsc{first-order steering}, a formulation of activation steering as a first-order approximation of weight update matrices parameterized by a vector $\vs$, and establish theoretical bounds on the approximation error of first-order steering. We then develop a novel model merging procedure, \textsc{HeRD-Merging}, which minimizes the first-order approximation error terms to enable higher first-order steering accuracy. Together, our method produces steering vectors that control both individual and composed behaviors more accurately than existing activation steering methods. Furthermore, HeRD-Merging matches the performance of conventional model-merging baselines, while producing weight adaptations that admit more accurate first-order steering vectors.
\end{abstract}

\section{Introduction}

Reliable identification of interpretable features in language models poses a central problem in mechanistic interpretability \citep{kim2018tcav, sharkey2025openproblemsmechanisticinterpretability}. Prior work has developed numerous strategies for supervised and unsupervised discovery of interpretable features in large language models (LLMs) \citep{turner2024steeringlanguagemodelsactivation, shu2025surveysparseautoencodersinterpreting}. A popular class of such methods is activation steering, which finds vector directions in the activation space of a language model that, when added to the residual stream, induce a particular behavior. Various activation steering methods---such as contrastive activation addition (CAA), representation engineering (RepE), sparse autoencoder (SAE) steering, and ODE steering (ODESteer)---have shown promise in exploiting interpretable features in the activations of LLMs to manipulate their behavior \citep{panickssery2024steeringllama2contrastive, fang-etal-2026-controllable, zou2023representationengineeringtopdownapproach, zhao2026odesteer}. However, utilizing these methods to control multiple target behaviors simultaneously remains a challenge \citep{oozeer2026linearsteeringunifiedmultiattribute, braun2025understanding}. In contrast, prior work in model merging has demonstrated that, using methods such as task arithmetic, TIES, DARE, and KnOTS, multiple target behaviors can be composed when represented as weight adaptations \citep{ilharco2023editingmodelstaskarithmetic, prabhakar2024lorasoupsmergingloras, kangaslahti2025continuous, stoica2024modelmergingsvdtie, yu2024language, yadav2023tiesmerging}. Work on RepE has also proposed procedures such as Low-Rank Representation Adaptation, which simulate activation-space interventions through corresponding parameter updates, suggesting a relationship between weight adaptation and activation steering. However, the converse, translating weight updates into activation steering, remains underexplored. Understanding the conditions under which weight adaptation can be converted into activation steering could extend existing model merging methods to activation steering, thus enabling simultaneous and interpretable inference-time control of multiple target behaviors.

Here we show that a differentiable model-merging method parameterized by steering strengths $\vs$ can be converted into activation steering by taking a first-order approximation of each hidden layer with respect to $\vs$. Each column of the resulting Jacobian matrix defines a direction in activation space, and these directions can be naturally composed through vector addition. We then show that the Hessian norm of each hidden-layer component bounds the approximation error of first-order steering, providing a theoretical measure of steering accuracy that existing methods lack. From this measure, we develop \textsc{HeRD-Merging}, a differentiable model-merging procedure that minimizes the first-order approximation error terms while preserving conventional model-merging performance. Empirically, first-order steering controls both individual and composed behaviors more accurately than existing activation steering methods, while HeRD-Merging produces weight adaptations that admit more accurate first-order steering vectors.

We make the following contributions:

\begin{enumerate}
\item We introduce \emph{first-order steering}, a method for translating differentiable model-merging procedures into activation steering. The Jacobian of each hidden layer with respect to the steering strengths yields composable activation-space directions, while the corresponding Hessian norms provide theoretical bounds on the first-order approximation error.

\item We introduce \emph{HeRD-Merging}, a differentiable model-merging procedure designed to control the error bounds governing first-order steering accuracy. By reducing the relevant Hessian-based error terms, HeRD-Merging produces weight adaptations which, when paired with first-order steering, produce more accurate activation steering vectors.

\item We evaluate first-order steering on both individual and composed behavioral control and compare HeRD-Merging with established model-merging methods. First-order steering outperforms existing activation steering methods, while HeRD-Merging improves first-order steering accuracy and maintains performance comparable to conventional model-merging baselines.

\end{enumerate}

\section{Preliminaries}

\paragraph{Activation Steering} Prior work in mechanistic interpretability has produced several steering methods that induce behaviors through additive updates to the intermediate activations of an LLM. Activation Addition (ActAdd) demonstrated that a single \emph{steering vector}, computed as the difference in residual-stream activations between a pair of contrastive prompts, can be added at a chosen layer to induce a targeted behavioral shift at inference time without any gradient-based training \citep{turner2024steeringlanguagemodelsactivation}. At inference time, a steering vector $\vv$ is applied through 
\begin{equation}
\label{equation:activation-addition}
\vh' = \vh + \alpha\vv
\end{equation}
for some steering strength parameter $\alpha \in \mathbb{R}$. A set of steering vectors $\vv_1, \dots, \vv_n$ can then be composed through vector addition so that the aggregate steering vector is $\vv = \sum_{k = 1}^n \vv_k.$  Through this principle, numerous strategies have emerged towards computing optimal steering vectors applied through \Cref{equation:activation-addition}. Contrastive Activation Addition (CAA) computes steering vectors from paired positive/negative examples but focuses on a systematic evaluation across many behaviors, showing that the resulting vectors transfer across prompt formats and can be combined with finetuning \citep{panickssery2024steeringllama2contrastive}. Representation Engineering (RepE) generalizes this idea by extracting relevant steering vectors using additional supervised and unsupervised methods such as PCA, $k$-means clustering, and logistic regression \citep{zou2023representationengineeringtopdownapproach}. Feature Guided Activation Additions (FGAA) leverages sparse-autoencoder features to select and combine directions, aiming to produce more interpretable and disentangled steering vectors than raw activation differences \citep{soo2025interpretable, cunningham2023sparseautoencodershighlyinterpretable, shu2025surveysparseautoencodersinterpreting}. Representation Surgery generalizes existing steering heuristics and reduces the computation of steering functions to a least-squares objective \citep{pmlr-v235-singh24d}. Activation Transport (AcT) reframes activation steering as an optimal-transport problem and learns a transport map that moves the distribution of source activations toward a target distribution \citep{rodriguez2024controllinglanguagediffusionmodels}.

These strategies have been shown to work in various in-distribution benchmarks, but have also demonstrated inconsistent results in out-of-distribution settings, especially when steering multiple behaviors simultaneously \citep{Scalena_2024, vanderweij2024extendingactivationsteeringbroad, braun2025understanding, tan2024analysing}. Moreover, existing progress in activation steering has only presented marginal improvements in accuracy, limiting its application in practical parameter-efficient finetuning requirements \citep{apolinario2026lancelowrankactivation, koike-akino2025quantumpeft}. More recent methods---such as ODESteer, MiMiC, RE-Control, and AcT---introduce novel paradigms of activation steering to increase accuracy but deviate in formulation from \Cref{equation:activation-addition}, meaning they lack a nontrivial procedure for composing generated steering vectors \citep{zhao2026odesteer, heyman2026steer, kong2024aligning, rodriguez2024controllinglanguagediffusionmodels}. In contrast, we show that first-order steering outperforms multiple activation steering paradigms while offering accurate composability of vectors in a manner identical to activation addition.

\paragraph {Model Interpolation and Merging} With sufficient data, low-rank adaptation (LoRA) provides an efficient method for finetuning a model \citep{hu2021loralowrankadaptation}. However, LoRA on its own does not provide a framework for merging or interpolating adapters. Consequently, techniques of \emph{model merging} are commonly applied to combine the capabilities of multiple independently finetuned LoRA adapters. One such technique—task arithmetic—defines \emph{task vectors} as the element-wise difference, $\Delta \mW$, between finetuned and base model weights. These vectors can then be added or subtracted to compose or suppress behaviors \citep{ilharco2023editingmodelstaskarithmetic}. In particular, for a collection of LoRA adapters $\Delta \mW_1, \dots, \Delta \mW_n,$ task arithmetic generates a unified LoRA adapter $$\Delta \mW(\vs) = \sum_{k = 1}^n s_k\Delta \mW_k$$ for some $\vs \in \mathbb{R}^n.$ This strategy can accurately produce simultaneous behavioral control but can be improved through various model merging methods. TIES-Merging trims low-magnitude parameters and resolves sign conflicts prior to aggregation \citep{yadav2023tiesmerging}. DARE sparsifies task vectors further by randomly dropping and rescaling delta parameters \citep{yu2024language}. KnOTS decomposes task vectors via SVD before merging, operating in a shared low-dimensional subspace to reduce cross-task interference and improve the composability of merged models through existing methods like TIES and DARE \citep{stoica2024modelmergingsvdtie}.

Model merging methods demonstrate that weight-space adaptations can be composed while retaining strong per-task performance, providing a natural foundation for constructing composable activation steering vectors using first-order steering. We therefore generalize the update rule of existing model merging methods to develop HeRD-Merging, a differentiable model merging procedure which encourages greater first-order steering performance while matching existing baselines in merging accuracy. 

\section{First-Order Steering}

In this section, we present our first-order steering methodology and establish theoretical bounds on its accuracy. Let $f_\theta$ denote a frozen autoregressive transformer with $L$ layers and hidden dimension $d$, and let $\mW_p^{\mathrm{base}}$ denote the frozen pretrained weight matrix of $f_\theta$ at projection head $p$. We denote the hidden layer of $f_\theta$ at layer $\ell$ by $\vh_\ell$, with the $i$th component denoted $h_{\ell, i}.$ 

Let $n$ be the number of target behaviors (i.e. the number of LoRA adapters) and let $g : \mathbb{R}^{d + n} \to \mathbb{R}^d$, with evaluation denoted $g(\vx, \vs)$ for vectors $\vx \in \mathbb{R}^d$ and $\vs \in \mathbb{R}^n.$ Then $J_\vs g$ (resp. $J_\vx g$) denotes the Jacobian matrix resulting from viewing $g$ as a function with $\vx$ (resp. $\vs$) fixed. Similarly, for scalar-valued $g_i$, $H g_i$ denotes the Hessian and $H_\vs g_i$ the Hessian with $\vx$ fixed.

\label{section:fo-steering}
Let $\mathcal{P}$ be the set of projection head indices and let $\Delta \mW_{1,p}, \dots, \Delta \mW_{n,p}$ be LoRA adapters applied at $p \in \mathcal{P}$. Then task arithmetic defines a parameterized update
\begin{equation}
    \Delta \mW_p(\vs) = \sum_{k = 1}^n s_k \Delta \mW_{k, p}, \qquad \mW_p(\vs) = \mW_p^{\mathrm{base}} + \Delta \mW_p(\vs).
\end{equation}
for some $\vs \in \mathbb{R}^n.$ Henceforth, we omit the indexing by $p$ for readability when it is not required. Let $f_\theta^{\mathrm{TA}}$ be the transformer with the task arithmetic updates added at each $p \in \mathcal{P}$, and let $\vh_\ell'$ be the hidden layer of $f_\theta^{\mathrm{TA}}$ at layer $\ell$, with the $i$th component denoted $h_{\ell, i}'$. Under this formulation, $\vh_{\ell}'$ admits the first-order approximation \begin{equation}
\label{equation:first-order-steering}
h_{\ell, i}'(\vx, \vs) = h_{\ell, i}(\vx) + \nabla_\vs h_{\ell, i}' (\vx, 0)^\top \vs + r_i(\vs),
\end{equation} where a bound on the remainder $\vr$ is given by Taylor's theorem
\begin{equation}
\label{equation:remainder-bound}
\left\|\vr(\vs)\right\| \leq \frac{1}{2} \sum_{i = 1}^d \sup_{\vs^* \in [\mathbf{0}, \vs]} \left\|H_\vs h_{\ell, i}'(\vx, \vs^*)\right\|_F \|\vs\|^2,
\end{equation}
with $H_\vs h_{\ell, i}'$ denoting the Hessian matrix of $h_{\ell, i}'$ with respect to $\vs$, $\|\cdot\|_F$ denoting the Frobenius norm for matrices, and $[\mathbf{0}, \vs] \subseteq \mathbb{R}^n$ denoting the line segment from $\mathbf{0}$ to $\vs$. Thus, with a sufficiently small Hessian of $h_{\ell, i}'(\vx, \vs),$ the inner product $\nabla_\vs h_{\ell, i}'(\vx, 0)^\top \vs$ yields activation addition terms 
\begin{equation}
\label{equation:first-order-steering-composition}
h_{\ell, i}'(\vx, \vs) \approx h_{\ell, i}(\vx) + \sum_{k = 1}^n \frac{\partial h_{\ell, i}'}{\partial s_k}(\vx, 0) s_k.
\end{equation} This is an instance of \Cref{equation:activation-addition}, besides the dependence of the activation addition term on the residual input $\vx$. To remove this dependency, we compute steering vectors for each target behavior optimal over some dataset of $M$ activation samples $\mathcal{D}$ so that $\mA^* = \argmin_\mA \frac{1}{M} \sum_{m = 1}^M \|J_\vs\vh_{\ell}' (\vx_m, 0) - \mA\|_F^2$. The analytical solution to this least-squares problem is given by the mean \begin{equation}
\label{equation:least-squares-solution}
\mA^* = \frac{1}{M} \sum_{m = 1}^M J_\vs\vh_{\ell}' (\vx_m, 0).
\end{equation} This formulation has the following advantages: (i) steering vectors individual to each target behavior can be read out directly from the columns of $\mA^*$; (ii) steering vector composition is given by summation through the matrix product $\mA^* \vs$; (iii) the derivation of steering vectors from LoRA adapters allows for the simultaneous computation of steering vectors for any layer where an adapter is placed. 

We now bound the error accumulated from \Cref{equation:first-order-steering-composition,equation:least-squares-solution}. Intuitively, controlling the variance of $J_\vs \vh_\ell'$ with respect to $\vx$ over some dataset reduces the dependency on the prompt and therefore should reduce the least squares error. This can be motivated through the following propositions. 

\begin{proposition}[Gaussian Poincar\'e inequality]
\label{theorem:gaussian-poincare}
Let $\rva = (\rva_1, \dots, \rva_a)$ be a vector of independent standard Gaussian random variables. If $\vg : \mathbb{R}^a \to \mathbb{R}^b$ is smooth, then $$\|\Cov(\vg(\rva))\|_F \leq \mathbb{E}[\|J \vg(\rva)\|_F^2].$$
\end{proposition}
The scalar case $\mathrm{Var}(g(\rva)) \leq \mathbb{E}\|\nabla g(\rva)\|^2$ is stated in \citep{chatterjee2007fluctuationseigenvaluessecondorder}. The form above follows by applying it to each coordinate of $\vg$ and using $\|\Cov(\vg(\rva))\|_F \leq \operatorname{trace}(\Cov(\vg(\rva)))$ for positive semidefinite matrices. We now relate the Jacobian to the least squares error (\ref{equation:least-squares-bound}).

\begin{proposition}
\label{theorem:error-bound-hessian}
Let $S \subseteq \mathbb{R}^{d + n}$ be an open rectangle containing $\mathcal{D} \times \{0\} \subseteq \mathbb{R}^{d + n}$ and define $Q = \sup_{(\vx_0, \vs_0) \in S}\|J_\vx J_\vs \vh_\ell'(\vx_0, \vs_0)\|_F$ and $L = \sup_{\vx_a, \vx_b \in \mathcal{D}} \|\vx_a - \vx_b\|.$ Then 
\begin{equation}
\label{equation:least-squares-bound}
\frac{1}{M} \sum^M_{m = 1} \|J_\vs\vh_{\ell}' (\vx_m, 0) - \mA^*\|^2_F \leq \frac{L^2Q^2}{2}.
\end{equation}
\end{proposition}

Since $Q \leq \sup_{(\vx, \vs) \in S} \sum_{i=1}^d \|H h_{\ell,i}' (\vx, \vs)\|_F$, the same Hessian norm that bounds the remainder in \Cref{equation:remainder-bound} also bounds $Q$. Regularizing the Hessian norm therefore simultaneously manages the first-order approximation error and the least squares error, suggesting a quantifiable signal of steering accuracy under our methodology.

\section{Hessian Regularized Differentiable Merging (HeRD-Merging)}

In this section, we present a differentiable model-merging procedure that controls the size of the bound (\ref{equation:remainder-bound}) and allows for the minimization of $\|J_\vx J_\vs \vh_\ell'(\vx, \vs)\|_F$ as in \Cref{theorem:error-bound-hessian}. In \Cref{section:experiments}, we show that this effort leads to improved steering performance while retaining merging performance. 

\subsection{The HeRD Adapter Function}

Let $\Delta \mW_1, \dots, \Delta \mW_n$ be a collection of LoRA adapters and let $\mathcal{C} = \{S \subseteq \{1, \dots, n\} : |S| \geq 2\}.$ Then a \emph{HeRD adapter} is a function 
\begin{equation}
\label{equation:herd-adapter}
\Delta \mW(\vs) = \sum_{k = 1}^n s_k \Delta \mW_k + \sum_{S \in \mathcal{C}} \eta_S \left(\prod_{t \in S} s_t \mR_S \right),
\end{equation} where $\mR_S = \frac{1}{r_S} \mB_S \mA_S$ for some chosen rank $r_S$ and $\eta_S, \mR_S$ are learnable parameters with $0 \leq \eta_S \leq \eta_{\mathrm{max}}$. The second term adds a scaled corrective low-rank matrix $\mR_S$ weighted by the intensity of each interaction $\prod_{t \in S} s_t.$ Fixing $\eta_S = 0$ recovers task arithmetic as a special case. Evidently, this makes $\Delta \mW$ smooth in $\vs$. Moreover, writing $\mR_S$ with a low-rank factorization allows for a constrained maximum rank of the update $\Delta \mW(\vs)$. Let $f^\mathrm{HeRD}_\theta$ be the model $f_\theta$ with the HeRD adapters $\Delta \mW_p(\vs)$ at each projection head index $p \in \mathcal{P}.$ Formulated as is, first-order steering with HeRD-Merging produces steering vectors equivalent to task arithmetic. This is handled by using the generalized first-order steering formulation of \Cref{equation:generalized-first-order-steering} described in \Cref{section:first-order-steering-generalizes-ta}. Importantly, this does not change the Hessian regularization results of \Cref{section:fo-steering}.

Optimizing the matrices $\mR_S$ for controlling the Hessian norms is computationally expensive, as it would require materializing multiple mixed Jacobian matrices through each transformer layer. Instead, we first fix $\eta_S = 1$ for all $S$ and train the matrices $\mR_S$ to encourage merging accuracy using gradient descent. Then, we fix the matrices $\mR_S$ and compute the optimal $\eta_S$, which reduces the problem to a nonlinear least squares objective. This penalizes the Hessian norm without computing any derivatives with respect to the matrices $\mR_S$. 

\subsection{Distribution Loss by Gradient Descent}

Fix an autoregressive context $\vc = (\vx, \vy_{<t})$ and let $\mathcal{V}$ be the vocabulary of the transformer $\vf_\theta^{\mathrm{HeRD}}$. We define $p_\vs(\vv\mid\vc)$ as the next-token probability assigned to $\vv \in \mathcal{V}$ given context $\vc$ and steering strengths $\vs$. In particular, let $p_0(\vv\mid\vc)$ and $p_k(\vv\mid\vc)$ denote the next-token probabilities for $\vs = \mathbf{0}$ and $\vs = \ve_k$ respectively. Then we define the interpolation of the distributions $p_k$ at $\vs$ as $$\log q_\vs(\vv \mid \vc) = \log p_0(\vv \mid \vc) + \sum_{k = 1}^n s_k(\log p_k(\vv \mid \vc) - \log p_0(\vv \mid \vc)).$$ Through normalization by the partition function $Z_\vs(\vc) = \sum_{\vv \in \mathcal{V}} q_\vs(\vv \mid \vc),$ we can form $q_\vs$ into a probability measure $Q_\vs(\vv \mid \vc) = q_\vs(\vv \mid \vc)/Z_\vs(\vc).$ From interpolation, this probability measure has the property that $Q_{\ve_k} = p_k$ and $Q_0 = p_0.$ We can now define the distribution-matching objective of our merging procedure over some grid $G \subseteq [0, 1]^n$ as $$\mathcal{L}_\textrm{dist}(R) = \mathbb{E}_{\vc \sim \mathcal{D}} \mathbb{E}_{\vs \sim U(G)}[KL(Q_\vs(\cdot \mid \vc), p_\vs(\cdot \mid \vc))],$$ where $U(G)$ denotes the uniform distribution over the set $G$ and $\mathcal{D}$ is some dataset of prompts. This objective can then be minimized through gradient descent for all layers simultaneously.

\subsection{Hessian Regularization by the Levenberg-Marquardt Algorithm}

Subsequently, we fix the matrices $\mR_S$ in \Cref{equation:herd-adapter} and optimize $\boldsymbol{\eta} = (\eta_S)_{S \in \mathcal{C}}$ by constructing a residual vector of objectives $r(\boldsymbol{\eta})$ and minimizing the nonlinear least squares objective \begin{equation}
\label{equation:nonlinear-least-squares-objective}
\min_{\boldsymbol{\eta} \in [0, \eta_{\mathrm{max}}]^{|\mathcal{C}|}} \frac 1 2 \|\vr(\boldsymbol{\eta})\|_2^2
\end{equation} with the Levenberg-Marquardt Algorithm (LMA) \citep{levenbergmarquardt}. For each $\boldsymbol{\eta},$ we define residuals $r_1(\boldsymbol{\eta}), r_2(\boldsymbol{\eta}), r_3(\boldsymbol{\eta}), r_4(\boldsymbol{\eta})$ and define $\vr(\boldsymbol{\eta}) = \sqrt{2} [\sqrt{r_1(\boldsymbol{\eta})} \, \sqrt{r_2(\boldsymbol{\eta})} \, \sqrt{r_3(\boldsymbol{\eta})} \, \sqrt{r_4(\boldsymbol{\eta})}].$

\paragraph{Distribution-matching.} We retain the distribution-matching objective of the previous stage by defining the residual $r_1 = \mathrm{KL}(Q_\vs(\cdot \mid \vc), p_\vs(\cdot \mid \vc)).$ This prevents divergence from distribution alignment meant to be achieved in the previous stage of training. 

\paragraph{Hessian regularization.} By the linearity of the derivative, it does not suffice to merely add the Jacobian matrix norm as a residual. Instead, we optimize the size of the mixed Jacobian matrix relative to the Jacobian matrix. Let $\vb \sim U(G)$ uniformly and let $\vs^*$ be the unit vector constructed by uniformly sampling from $\{1, -1\}^n$ and then scaling by $\frac{1}{\sqrt{n}}$. We define the residual $$r_2 = \frac{\|J_\vs (J_\vs \vh_\ell (\vx, \vb) \cdot \vs^*) \cdot \vs^*\|}{\|J_\vs \vh_\ell (\vx, \vb) \cdot \vs^*\|}.$$ Furthermore, the mixed partial derivatives with respect to $\vx$ can be directly controlled through the total variance over a dataset $\mathcal{D}$ of activations produced by various prompts as $$r_3 = \frac{\|\mathrm{Cov}_{\vx \sim \mathcal{D}}[J_\vs \vh_\ell (\vx, \vb) \cdot \vs^*]\|_F}{\mathbb{E}[\|J_\vs \vh_\ell (\vx, \vb) \cdot \vs^*\|^2]}.$$ By \Cref{equation:remainder-bound} and \Cref{theorem:error-bound-hessian}, $r_2$ and $r_3$ work to cheaply penalize approximations of relevant blocks of $\|H h_{\ell, i}\|_F.$ 

\paragraph{Correction term magnitude.} The parameters $\eta_S$ also serve to control the magnitude of the matrices $\mR_S$ relative to the LoRA adapters. So, we add the residual $$r_4 = \sum_{S \in \mathcal{C}} \frac{\eta_S \|\mR_S\|_F^2}{\sum_{k = 1}^n \|\Delta \mW_k\|_F^2}.$$ This allows $\boldsymbol{\eta}$ to simultaneously prevent the matrices from overcorrecting the learned behavioral effect of the LoRA adapters. Since each residual component is dependent on $\boldsymbol{\eta}$, this supplies the function $\vr(\boldsymbol{\eta})$ for the nonlinear least-squares problem \Cref{equation:nonlinear-least-squares-objective}. Each minibatch of the LMA resamples the JVP direction $\vs^*$ to increase the coverage in the interior of $[0, 1]^n.$

\section{Experiments}
\label{section:experiments}

Our methodology aims to facilitate the computation of accurate and composable steering vectors across multiple target behaviors and models. First-order steering offers a framework for computing these steering vectors, while HeRD-Merging is constructed to produce more accurate steering vectors when paired with first-order steering. In this section, we evaluate first-order steering against existing activation steering methods and demonstrate that HeRD-Merging improves first-order steering performance relative to task arithmetic. As a secondary result, we also show that HeRD-Merging remains competitive with existing model merging approaches, supporting its practical use as an intermediate between a collection of LoRA adapters and their corresponding first-order steering vectors.

\subsection{Setup and Training Details}
\label{section:experiment-setup}

To evaluate our techniques jointly, we train LoRA adapters on steering directions which are semantically independent of each other so that the accuracy of merging adapters and steering vectors can be tested meaningfully. The directions should also be a reasonably established use case of activation steering, so that they can be benchmarked against other steering methods. For this, we choose three distinct behavioral manipulations in the model's output style:

\begin{enumerate}
  \item{\textbf{Analogy Usage:} The model is trained to answer questions by constructing a heavy, overly complicated analogy and responding in the context of this analogy.}
  \item{\textbf{Bulletpoint Usage:} The model is trained to format all of its replies in bulletpoints, irrespective of how appropriate it is for the particular user query.} 
  \item{\textbf{Sophisticated Language:} The model is trained to speak in highly sophisticated, exaggerated, and archaic language.} 
\end{enumerate}

Because the model can measurably exhibit any combination of analogy use, bulletpoint use, and sophisticated language, these directions have meaningful semantic composability. They are also sufficiently complex for both activation steering and model merging, while being in line with the benchmarking methods of prior works in activation steering \citep{11224465}. To train the models, we first synthetically generate data for all of the directions using GPT-5.5. Generation proceeds by providing the model with a user conversation from the UltraChat dataset \citep{ding2023enhancingchatlanguagemodels} and instructing GPT-5.5 to re-generate model responses to user prompts according to the desired behavioral direction. The resulting datasets each consist of approximately $10,000$ samples. Examples of synthetically generated outputs for each target behavior are available in \Cref{section:dataset-examples}.

We report our experimental results for three model families of varying sizes. We run matched experiments on Qwen2.5-VL-3B, Gemma-4-E2B, and Mistral-7B-Instruct-v0.3 \citep{bai2025qwen25vltechnicalreport, gemmateam2026gemma4technicalreport, jiang2023mistral7b}. We train a set of adapters for each behavioral direction independently. For each model, we freeze the base model and attach LoRA adapters to each projection head of the model with LoRA rank $8$. Adapters are trained using the Adam optimizer \citep{kingma2017adammethodstochasticoptimization} for $100$ steps with a minibatch size of $2$, $16$ gradient accumulation steps, and a sequence length of $512$.

\subsection{DeBERTa-v3 Classifiers}

One advantage of the chosen steering directions is that they are suitable for computationally inexpensive evaluation by a classification model. Because conversations in our dataset are generated from UltraChat, we use those as labeled negative examples and the synthetically generated conversations as labeled positive examples. We split this dataset into $80\%$ training, $10\%$ validation, and $10\%$ test. The training data is then used to finetune the DeBERTa-v3-base model, which has shown reliable agreement with human evaluators on classification tasks once trained \citep{he2023debertav, li2024llmsasjudgescomprehensivesurveyllmbased}. Due to the lexical distinctness of each steering direction, the classifiers achieve near-perfect accuracy on the test sets (containing $1994$ samples). We also compute thresholds for binary classification using the validation accuracy as the heuristic. Detailed classifier results are reported in \Cref{section:classifier-information}.

\subsection{Activation Steering Experiments}
\label{section:steering-results}

In principle, first-order steering can be applied at all layers where there are weight updates. This is advantageous because it can lead to greater performance and higher expressivity. However, because baseline activation steering methods exhibit output degradation when concurrently applied at too many model layers, we conservatively follow the evaluation approaches common in prior work and identify a single common layer reasonable to use for all of the methods \citep{zhao2026odesteer, sharma2026coldsteer}. For each target behavior $k$, we maximize the normalized mean activation difference generated by contrastive activation distributions $\mathcal{D}^+_k, \mathcal{D}^-_k$ given by $$S_{\ell,k} = \frac{\left\|\mathbb E_{\vx \sim \mathcal{D}^+_k}[h_{\ell}(x)] - \mathbb E_{\vx \sim \mathcal{D}^-_k}[h_\ell(x)] \right\|^2}{\sum_i\operatorname{Var}_{\vx \sim \mathcal{D}^+_k}(h_{\ell,i}(x)) + \sum_i\operatorname{Var}_{\vx \sim \mathcal{D}^-_k}(h_{\ell,i}(x))}$$ for each layer $\ell$. We then choose the optimal layer by $\ell^*_k = \argmax_\ell S_{\ell, k}.$ We train HeRD adapters at that layer by sampling from a uniformly spaced grid $G \subseteq [0, 1]^n$ and training the matrices $\mR_S$ for $100$ gradient descent steps using the Adam optimizer \citep{kingma2017adammethodstochasticoptimization}. The LMA is then performed with a mini-batch size of $1$, $\eta_{\mathrm{max}} = 1$, and an initial damping parameter of $\mu_0 = 0.1$ for $12$ steps. We then compile a separate set of test prompts independent of the synthetically generated datasets used to train the LoRA adapters. These test prompts are used to produce $100$ outputs from the model for each of the $2^n$ corners of $[0, 1]^n$, resulting in $800$ outputs per steering method. These outputs are then evaluated by the DeBERTa-v3 classifier to produce behavioral adherence rates for the merged models. We also report confidence intervals using non-parametric percentile bootstrapping on $2000$ prompt-level resamples, computing the accuracies for all $8$ combinations of steering directions with each resample. Finally, we take the mean of the Jacobian in the first-order approximation over user prompts from the UltraChat training dataset to obtain the steering vector from first-order steering.

\FloatBarrier

\begin{table}[!ht]
\centering
\caption{Overall behavioral adherence (\%) of activation steering methods, averaged across all three models. Steering is applied at a single common layer for all methods. Adherence for each behavior is computed across all combinations of target behaviors (single, pairs, and all three) that include that behavior. We report it individually for each behavior in its corresponding column as well as on average (right-most column). $\dagger$ denotes our methods. Bold indicates the best result per column.}
\footnotesize
\setlength{\tabcolsep}{4pt}
\begin{tabular}{@{}lcccc@{}}
\toprule
\textbf{Steering Method} & \textbf{Analogy Use} & \textbf{Bulletpoint Use} & \shortstack{\textbf{Sophisticated}\\\textbf{Language}} & \textbf{Average} \\
\midrule
First-Order (HeRD-Merging)$^\dagger$    & 19.8 & \textbf{65.9} & \textbf{53.6} & \textbf{46.4} \\
First-Order (Task Arithmetic)$^\dagger$ & \textbf{20.2} & 63.5 & 44.2 & 42.6 \\
\midrule
ODESteer    & 0.2 & 37.8 & 31.5 & 23.2 \\
Linear-AcT  & 1.1 & 45.8 & 44.1 & 30.3 \\
RE-Control  & 0.4 & 32.6 & 5.7  & 12.9 \\
CAA         & 2.1 & 42.9 & 37.3 & 27.4 \\
RepE        & 2.6 & 42.9 & 26.9 & 24.1 \\
\bottomrule
\end{tabular}
\label{table:overall-activation-steering-adherence}
\end{table}

First-order steering is compared to CAA, Linear-AcT, RE-Control, RepE, and ODESteer \citep{panickssery2024steeringllama2contrastive, rodriguez2024controllinglanguagediffusionmodels, kong2024aligning, zou2023representationengineeringtopdownapproach, zhao2026odesteer}. Similarly to the merging experiment, we ensure a fair comparison by using the best-performing hyperparameters described in the paper that introduces each method. CAA computes steering vectors by taking the mean differences in contrastive positive and negative pairs, provided by our synthetically generated datasets. Linear-AcT fits an input-dependent linear map to each activation to align the coordinate-wise observed activation distributions for positive prompts. RE-Control fits a small multi-layer perceptron (MLP) with output in $[0, 1]$ for $200$ steps of the Adam optimizer \citep{kingma2017adammethodstochasticoptimization}. The MLP predicts whether a particular behavior is met and computes an input-dependent steering vector through $30$ online gradient descent steps, using the MLP as the minimization objective. RepE computes steering vectors as the first principal component of centered mean differences between positive and negative pairs. ODESteer uses $10$ steps of Euler's method along the gradient of a log-density ratio barrier function using a degree-two polynomial count sketch with $\gamma = 0.1, c_0 = 1.0$.

Aggregating the results across all three models in \Cref{table:overall-activation-steering-adherence}, first-order steering with HeRD-Merging achieves an average behavioral adherence of $46.4\%$, compared with $42.6\%$ for first-order steering with task arithmetic. The strongest non-first-order baseline is Linear-AcT, at $30.3\%$, which is $16.1$ percentage points away from first-order steering with HeRD-Merging. First-order steering also obtains the highest aggregate adherence for all three axes, exceeding the strongest non-first-order results by $17.6$, $20.1$, and $9.5$ percentage points, respectively. 

We also present more detailed adherence results for every model in \Cref{table:activation_steering_adherence}. Here too first order steering with HeRD-Merging achieves the highest average behavioral adherence when averaging across the three target behaviors on all three models (right-most column). $64.7\%$ on Qwen2.5-VL-3B, $42.5\%$ on Mistral-7B-Instruct-v0.3, and $32.0\%$ on Gemma-4-E2B. These exceed the strongest non-first-order baseline for each model, Linear-AcT ($25.4\%$) on Qwen2.5, ODESteer ($39.2\%$) on Mistral, and Linear-AcT ($31.3\%$) on Gemma, by $39.3$, $3.3$, and $0.7$ percentage points respectively. First-order steering with task arithmetic also exceeds these baselines in mean adherence, reaching $55.4\%$, $40.7\%$, and $31.8\%$, respectively. Relative to task arithmetic, HeRD-Merging improves mean adherence by $9.3$, $1.8$, and $0.2$ percentage points respectively. Although the benefit is most pronounced on Qwen, the overall benefit of the trained terms $\mR_S$ is supported by our experiments. Note that while individual baselines are sometimes stronger for a particular model and target behavior, no baseline method is consistently second best. Additional results for first-order steering applied at all layers using both HeRD-Merging and task arithmetic are reported in \Cref{section:all-layer-experiments}.

\begin{table*}[t]
\centering
\caption{The detailed behavioral adherence (\%) results of \Cref{table:overall-activation-steering-adherence} expanded and reported for each model with its 95\% confidence intervals.}
\footnotesize
\setlength{\tabcolsep}{4pt}
\begin{tabular}{@{}llcccc@{}}
\toprule
\textbf{Model} & \textbf{Steering Method} & \textbf{Analogy Use} & \textbf{Bulletpoint Use} & \shortstack{\textbf{Sophisticated}\\\textbf{Language}} & \textbf{Average} \\
\midrule
\multirow{7}{*}{\rotatebox[origin=c]{90}{Qwen2.5-VL-3B}}
 & First-Order (HeRD-Merging)$^\dagger$       & 42.2 [36.8, 48.0] & \textbf{74.2 [68.2, 80.0]} & \textbf{77.8 [73.0, 82.5]} & \textbf{64.7} \\
 & First-Order (Task Arithmetic)$^\dagger$    & \textbf{43.0 [38.2, 47.8]} & 67.0 [61.0, 73.0] & 56.2 [49.5, 63.2] & 55.4 \\
\cmidrule(l){2-6}
 & ODESteer                       & 0.0 [0.0, 0.0]  & 21.5 [14.0, 29.5]          & 5.5 [2.5, 9.2]    & 9.0  \\
 & Linear-AcT                     & 1.6 [1.0, 2.2]  & 29.5 [27.0, 32.1]          & 45.1 [42.1, 48.1] & 25.4 \\
 & RE-Control                     & 0.2 [0.0, 0.5]  & 30.6 [27.1, 34.1]          & 3.2 [1.9, 4.6]    & 11.3 \\
 & CAA                            & 5.3 [4.1, 6.5]  & 25.9 [23.3, 28.6]          & 42.4 [39.2, 45.8] & 24.5 \\
 & RepE                           & 0.2 [0.0, 0.5]  & 29.6 [26.3, 33.1]          & 3.4 [2.2, 4.8]    & 11.1 \\
\midrule
\multirow{7}{*}{\rotatebox[origin=c]{90}{\shortstack{Mistral-7B \\ Instruct-v0.3}}}
 & First-Order (HeRD-Merging)$^\dagger$ & 16.6 [15.3, 17.8] & 63.5 [60.7, 66.5] & 47.5 [44.3, 50.8] & \textbf{42.5} \\
 & First-Order (Task Arithmetic)$^\dagger$    & \textbf{17.1 [15.8, 18.3]} & 63.7 [61.0, 66.5] & 41.3 [38.7, 44.0] & 40.7 \\
\cmidrule(l){2-6}
 & ODESteer                       & 0.3 [0.0, 0.8]  & \textbf{70.8} [67.0, 74.7]          & 46.4 [42.2, 50.5] & 39.2 \\
 & Linear-AcT                     & 0.8 [0.4, 1.4]  & 50.8 [47.3, 54.4]          & 51.4 [47.6, 55.4] & 34.3 \\
 & RE-Control                     & 0.2 [0.0, 0.6]  & 33.2 [29.2, 37.2]          & 4.8 [3.0, 6.8]    & 12.7 \\
 & CAA                            & 0.5 [0.1, 1.1]  & 58.7 [54.9, 62.3]          & 35.5 [32.1, 39.0] & 31.6 \\
 & RepE                           & 5.5 [4.5, 6.5]  & 54.4 [51.1, 57.5]          & \textbf{53.8 [50.1, 57.4]} & 37.9 \\
\midrule
\multirow{7}{*}{\rotatebox[origin=c]{90}{Gemma-4-E2B}}
 & First-Order (HeRD-Merging)$^\dagger$ & 0.5 [0.1, 1.1]  & \textbf{60.1 [56.4, 63.8]} & 35.4 [32.4, 38.7] & \textbf{32.0} \\
 & First-Order (Task Arithmetic)$^\dagger$    & 0.5 [0.1, 1.0]  & 59.9 [56.4, 63.6]          & 35.1 [31.9, 38.5] & 31.8 \\
\cmidrule(l){2-6}
 & ODESteer                       & 0.3 [0.0, 0.8]  & 21.1 [19.7, 22.5]          & 42.5 [38.6, 46.6] & 21.3 \\
 & Linear-AcT                     & 1.0 [0.4, 1.6]  & 57.1 [53.4, 60.9]          & \textbf{35.9 [32.7, 39.2]} & 31.3 \\
 & RE-Control                     & 0.8 [0.2, 1.6]  & 34.0 [29.8, 38.2]          & 9.0 [6.6, 11.6]   & 14.6 \\
 & CAA                            & 0.5 [0.1, 1.1]  & 44.0 [40.9, 46.8]          & 34.1 [30.9, 37.2] & 26.2 \\
 & RepE                           & \textbf{2.0 [1.2, 2.9]}  & 44.7 [40.8, 48.5] & 23.5 [20.8, 26.2] & 23.4 \\
\bottomrule
\end{tabular}
\label{table:activation_steering_adherence}
\end{table*}

\subsection{HeRD-Merging Comparisons}
\label{section:merging-comparisons}

We additionally compare HeRD-Merging against task arithmetic \citep{ilharco2023editingmodelstaskarithmetic}, DARE with task arithmetic for merging \citep{yu2024language}, DARE with TIES for merging \citep{yadav2023tiesmerging}, and KnOTS with TIES for merging \citep{stoica2024modelmergingsvdtie}. DARE randomly sparsifies each LoRA adapter with a drop rate $p = 0.9$. KnOTS performs SVD on the entire set of adapters to find an orthogonal basis for the weight updates. Merging through TIES is done by retaining only the parameters in each adapter which are in the top $20\%$ by absolute relative magnitude, electing the sign to be the sign of the total parameter mass, and merging with only the parameters that match this elected sign. For HeRD-Merging, we train HeRD adapters at all projection heads in all layers using the same hyperparameters as \Cref{section:steering-results}, with the evaluation and confidence intervals computations performed identically. To ensure a fair comparison, all hyperparameters are set to the default, best-performing configuration described by the main paper for each method. 

As shown in \Cref{table:overall-weight-merging-adherence}, HeRD-Merging achieves the highest overall behavioral adherence, at $79.9\%$, compared with $78.5\%$ for DARE task arithmetic, the strongest baseline, and $78.2\%$ for standard task arithmetic. This corresponds to improvements of $1.4$ and $1.7$ percentage points, respectively. HeRD-Merging also attains the highest aggregate adherence for each target behavior: $74.8\%$ for analogy use, $94.4\%$ for bulletpoint use, and $70.6\%$ for sophisticated language. Relative to the strongest baseline for each behavior, these scores improve upon task arithmetic by $0.6$ percentage points for analogy use, task arithmetic and DARE-TIES by $0.1$ percentage points for bulletpoint use, and DARE-TIES by $1.5$ percentage points for sophisticated language. These relatively small margins indicate that, in addition to improving first-order steering performance, HeRD-Merging maintains performance comparable to established model-merging procedures. We provide more detailed per-model results in \Cref{table:weight_merging_adherence}.

\FloatBarrier

\begin{table}[!ht]
\centering
\caption{Overall behavioral adherence (\%) of model merging methods, averaged across all three models using the same evaluation method as in \Cref{table:overall-activation-steering-adherence}. Bold indicates the best result per column.}
\footnotesize
\renewcommand{\arraystretch}{1.17}
\setlength{\tabcolsep}{4pt}
\begin{tabular}{@{}lcccc@{}}
\toprule
\textbf{Model Merging Method} & \textbf{Analogy Use} & \textbf{Bulletpoint Use} & \shortstack{\textbf{Sophisticated}\\\textbf{Language}} & \textbf{Average} \\
\midrule
HeRD-Merging (Ours)  & \textbf{74.8} & \textbf{94.4} & \textbf{70.6} & \textbf{79.9} \\
\midrule
DARE task arithmetic & 73.5 & 93.3 & 68.7 & 78.5 \\
Task arithmetic      & 74.2 & 94.3 & 66.1 & 78.2 \\
DARE-TIES            & 70.0 & 94.3 & 69.1 & 77.8 \\
KnOTS-TIES           & 51.2 & 90.5 & 52.2 & 64.6 \\
TIES                 & 51.4 & 76.1 & 55.1 & 60.9 \\
\bottomrule
\end{tabular}
\label{table:overall-weight-merging-adherence}
\end{table}

\FloatBarrier

\section{Conclusion}

We propose first-order steering, a reformulation of activation steering as a first-order approximation of model weight adaptations. We devise HeRD-Merging to merge multiple models while controlling the Hessian matrix norm of each transformer layer, which we show bounds the steering approximation error. First-order steering thus provides a quantitative indicator of performance that existing steering methods lack. In experiments, first-order steering improves overall accuracy and consistency across model families compared to baselines on behavioral adaptation tasks. HeRD-Merging performs similarly to established model merging methods while improving the resulting first-order steering performance. A practical limitation of first-order steering is the compute and data requirements of model finetuning. Future work could explore whether low-data finetuning techniques can produce comparable results. Notwithstanding, first-order steering also offers potential applications in mechanistic interpretability. For instance, exploring relationships between the Jacobian matrix $J_\vs \vh_\ell$ and the influence of a system prompt could reveal insights into model robustness and behavior prediction. Finally, the effectiveness of HeRD-Merging for Hessian regularization depends on the supplied weight update matrices. Future work could explore whether incorporating Hessian regularization into model training improves merging and steering performance.

\section*{Ethics Statement}

\ificlrpreprint\else This research adheres to the ICLR Code of Ethics. \fi
We do not use human subjects or user-sensitive data in our experiments. All data and models used for our experiments are publicly available. The aim of this work is to aid in the construction of aligned and controllable large language models. Our experiments demonstrate that our methods can cheaply induce various benign model behaviors. However, we acknowledge that these techniques are dual-use, and we promote responsible deployment in safety-critical environments. 

\section*{Reproducibility Statement}

\newcommand{\publiccontent}{%
  The results of our research can be reproduced using the code available in our \href{https://github.com/sripkunda/first-order-steering}{GitHub repository}. Our synthetically generated datasets are available at this \href{https://huggingface.co/collections/ChapAF/first-order-steering-datasets}{Hugging Face repository}. The checkpoints for the models we have finetuned on our synthetic datasets are available at this \href{https://huggingface.co/collections/ChapAF/first-order-steering-models}{Hugging Face repository}. For our experiments, we use four NVIDIA A100 GPUs with 80 GB VRAM. Further details on experiments are provided in \Cref{section:full-experiment-details}.
}

\publiccontent

\section*{AI Use Disclosure Statement}

In this work, we used LLMs to generate synthetic datasets, provide feedback on research methodology, and implement portions of our methodology. We additionally used AI tools to search for relevant information and analyze existing literature. We did not use generative tools to help develop theoretical models or conceptual frameworks, formulate mathematical claims, provide critical ingredients for proving mathematical claims, assist in the writing of proofs, propose or refine hypotheses, design research methodology, assist with translation, clean or reformat data, support qualitative and thematic data analysis, or interpret results. We have manually revised all AI-assisted work. Programs generated by LLMs were reviewed manually for consistency with our methodology, and we used software tests to verify intended functionality. All LLM feedback was thoroughly and manually reviewed for inconsistencies and compatibility with our methods. We take responsibility for the final content of this work, including text, claims, and experimental results presented with the aid of LLMs.

\ificlrpreprint
\section*{Author Contributions}

\textbf{Sri Pranav Kunda:} Conceptualization, Methodology, Software, Validation, Formal analysis, Investigation, Data Curation, Writing - Original Draft, Writing - Review \& Editing. \textbf{Alexander Kurz:} Conceptualization, Formal analysis, Resources, Writing - Review \& Editing, Supervision. \textbf{Tomáš Dominik:} Conceptualization, Resources, Writing - Review \& Editing, Supervision, Project administration. \textbf{Uri Maoz:} Conceptualization, Resources, Writing - Review \& Editing, Supervision, Project administration, Funding acquisition. 
\fi

\bibliography{references}
\bibliographystyle{iclr2027_conference}

\newpage 

\appendix

\startcontents[appendices]
\printcontents[appendices]{}{1}{\section*{Appendix Contents}}

\section{Theoretical Details}

\subsection{Proofs of Theorems}

Here, we supply a proof of the theoretical result stated in \Cref{theorem:error-bound-hessian}. The argument leverages the mean value inequality, which we restate below. 

\begin{proposition}[\cite{hörmander2015analysis}]
Let $V$ be a Banach space with norm $\|\cdot\|_V$, let $I$ be open in $\mathbb{R}$, and let $f : I \to V.$ Then $$\|f(y) - f(x)\|_V \leq \|y - x\| \sup_{x \in I}\|J_\vx f\|_V$$ for any $x, y \in I$.
\end{proposition}

With minimal effort, this result can be extended to an open rectangle $S$ and functions $f : \mathbb{R}^{d + n} \to V$ for $x, y \in S$. We now prove the theorem.

\begin{proof}[Proof of \Cref{theorem:error-bound-hessian}]
Let $g(\vx_a) = J_\vs \vh_\ell'(\vx_a,\mathbf{0})$ and let $\vx_a, \vx_b \in \mathcal{D}$. Since $S$ is an open rectangle, the mean value inequality gives $$\sum_{a,b} \|g(\vx_b) - g(\vx_a)\|^2_F \leq \sum_{a,b} Q^2\|\vx_b-\vx_a\|^2 \leq M^2Q^2L^2.$$
Let $$\mA^* = \frac{1}{M} \sum_{m=1}^M g(\vx_m).$$

Then we achieve the desired bound by
\begin{align}
\frac{1}{M} \sum_{m=1}^M \|g(\vx_m) - \mA^*\|^2_F &=
\frac{1}{2M^2} \sum_{a,b} \|g(\vx_b) - g(\vx_a)\|^2_F \\ &\le
\frac{Q^2L^2}{2}.
\end{align}
\end{proof}

\subsection{First-Order Steering Generalizes Task Arithmetic}
\label{section:first-order-steering-generalizes-ta}

In this section, we describe how we handle the equivalence of first-order and task-arithmetic steering vectors under the formulation of \Cref{equation:first-order-steering}. 

In the development of a procedure $\Delta \mW(\vs)$ for merging models given by weight updates $\Delta \mW_1, \dots, \Delta \mW_n,$ it is reasonable to expect that the partial derivative of $\Delta \mW(\vs)$ with respect to each component $s_k$ yields the original matrix $\Delta \mW_k$. In particular, we enforce 
\begin{equation}
\label{equation:merging-condition}
J_\vs \Delta \mW(\mathbf{0})\ve_k = \Delta \mW_k.
\end{equation} In such a procedure, the analytical derivation $J_\vs \vh_\ell(\vx, \mathbf{0})\, \vd$ is
\begin{align*}
J_\vs \vh_\ell(\vx, \mathbf{0}) \, \vd = \frac{d}{d\alpha} \vh_{\ell}(\vx, \mW_{p_1}(\alpha\vd), \dots, \mW_{p_q}(\alpha\vd))
\end{align*} for the relevant projection heads $p_1, \dots, p_q \in \mathcal{P}.$ Then the chain rule gives 
\begin{align}
\label{equation:jvp-formulation-of-jacobian}
J_\vs \vh_\ell(\vx, \mathbf{0}) \, \vd 
&= \sum_{r = 1}^q \frac{\partial \vh_\ell}{\partial \mW_{\ell, p_r}} (\vx, \mathbf{0}) [J_\vs \mW_{\ell, p_r} (\vs) \vd]
\\&= \sum_{r = 1}^q \frac{\partial \vh_\ell}{\partial \mW_{\ell, p_r}}(\vx, \mathbf{0}) \left[\sum_{k = 1}^n d_k \Delta \mW_{k, p_r}\right].
\end{align} Letting $\mW_\ell = (\mW_{\ell, p_1}, \dots, \mW_{\ell, p_q})$ and $\Delta \mW_{\ell, k} = (\Delta \mW_{k, p_1}, \dots, \Delta \mW_{k, p_q})$, we have \begin{equation}
\label{equation:first-order-steering-vector}
J_\vs \vh_\ell(\vx, \mathbf{0}) \vd = \frac{\partial \vh_\ell}{\partial \mW_\ell}(\vx, \mathbf{0}) \left[\sum_{k = 1}^n d_k \Delta \mW_{\ell, k}\right].
\end{equation} A slight generalization involves replacing $\vh_\ell$ with any block $\mF_i(\vx, \mW)$ that follows $\mW$ in the computational graph of the language model. In particular, any differentiable merging procedure which satisfies \Cref{equation:merging-condition} must contain task arithmetic as its first-order term, given by choosing $\mF_i = \mW$. This supports the naturality of task arithmetic as a rule for composing behaviors, in that behavioral updates at any $\mF_i$ can be viewed as applying the map $\frac{\partial \mF_i}{\partial \mW_{\ell}}(\vx, \mathbf{0})$ to task arithmetic. 

Importantly, \Cref{equation:first-order-steering-vector} motivates the generalization of \Cref{equation:first-order-steering} to \begin{equation}
\label{equation:generalized-first-order-steering}
\vh_{\ell}'(\vx, \vs) \approx \vh_{\ell}(\vx) + J_{\mW_\ell} \vh_{\ell}' (\vx, \mathbf{0}) [\Delta \mW(\vs)].
\end{equation}

By characterizing $J_{\mW_\ell} \vh_{\ell}' (\vx, \mathbf{0}) [\Delta \mW(\vs)]$ as a function of $\vs$, the same method of calculating steering vectors presented in \Cref{section:fo-steering} can be applied. Moreover, by \Cref{equation:first-order-steering-vector}, task arithmetic provides the same first-order steering rule as \Cref{equation:first-order-steering}. When used with HeRD-merging, this generalization allows for incorporating the matrices $\mR_S$. In the case of HeRD-Merging, composition is also preserved through addition of steering vectors for each target behavior and composition factor $\prod_{t \in S} s_t.$

\section{Experiment Details}
\label{section:full-experiment-details}

In this section, we provide further details for the experiments described in \Cref{section:experiments}.

\subsection{Synthetic Dataset Generation}
\label{section:synthetic-dataset-generation}
The generation of synthetic data was done separately for each of the three target behaviors using GPT-5.5. The sampling parameters are kept as the default from OpenAI (temperature $\tau = 1.0$ and top-$p$ with $p = 1.0$). Prompts from the UltraChat dataset are given to the model with a system prompt specifying the desired transformation. The original model responses are not provided to the model, and the model can only view the ongoing conversation history given by each row of the dataset. Each dataset transforms the same set of prompts, matched using the prompt IDs in UltraChat. 

\subsection{Classifier Training}
\label{section:classifier-information}

For training classifiers, we use the pretrained checkpoint available at \href{microsoft/deberta-v3-base}{this Hugging Face repository}. The model contains $12$ layers, hidden dimension $768$, and vocabulary size $128000$, amounting to $184$ million total parameters. We finetuned the base model to produce three variants corresponding to each target behavior. Training was done on the whole dataset of approximately $10000$ samples over $2$ epochs with batch size $128$ using the AdamW optimizer defaults ($\beta_1 = 0.9, \beta_2 = 0.999, \epsilon = 10^{-8}, \eta = 2 \times 10^{-5}$). We additionally schedule the learning rate $\eta$ to decay linearly with a $10\%$ warmup spanning $25$ steps. To prevent overfitting, we use a weight decay of $0.01$. The classification results are available in \Cref{table:deberta-results}. To further ensure that the classifiers do not mislabel the results of experiments, we manually evaluate a subset of all experiment data and verify agreement with classifier labels.

\begin{table}[h]
\centering
\caption{Finetuned DeBERTa-v3-base behavior classifier performance. Thresholds are selected on the validation set to maximize validation accuracy and are used for all evaluations.}
\small
\setlength{\tabcolsep}{5pt}
\begin{tabular}{lccccc}
\toprule
\textbf{Axis} & \textbf{Threshold} & \textbf{Test AUROC} & \textbf{Test AUPRC} & \textbf{Test Acc (\%)} & \textbf{Test $F_1$} \\
\midrule
Analogy Use                & 0.85 & 1.00 & 1.00 & 99.95 & 1.00 \\
Bulletpoint Use         & 0.12 & 1.00 & 1.00 & 99.85 & 1.00 \\
Sophisticated Language & 0.20 & 1.00 & 1.00 & 99.85 & 1.00 \\
\bottomrule
\end{tabular}
\label{table:deberta-results}
\end{table}

\subsection{LLM Generation Settings}

Generations for all evaluations across all methods were performed with no system prompt using greedy decoding for reproducibility. We include the Hugging Face repository for each model below.

\begin{enumerate}
  \item{\textbf{Qwen2.5-VL-3B:} \texttt{Qwen/Qwen2.5-VL-3B-Instruct}}
  \item{\textbf{Gemma-4-E2B:} \texttt{google/gemma-4-E2B}}
  \item{\textbf{Mistral-7B-Instruct-v0.3:} \texttt{mistralai/Mistral-7B-Instruct-v0.3}}
\end{enumerate}

We also note that the choice of these models supplies a range of parameter counts in addition to model architectures. Qwen2.5-VL-3B has 3 billion parameters, Gemma-4-E2B is a per-layer embeddings model with 5 billion total parameters, and Mistral-7B-Instruct-v0.3 has 7 billion parameters.

The computation of bootstrap confidence intervals for generations uses the nonparametric bootstrapping method, but with the modification that prompts, rather than individual generations, are resampled. This preserves the dependence among generations derived from the same prompt. For each bootstrap replicate, we sample $N$ prompts with replacement from the $N$ evaluation prompts, include every associated generation for each selected prompt, and record the classifier outputs. We repeat this procedure $B = 2000$ times and report the empirical $2.5$th and $97.5$th percentiles as a $95\%$ confidence interval for each reported statistic.

\section{Additional Results}

In this section we report supplementary results for first-order steering and HeRD-Merging.

\subsection{HeRD-Merging Per-Model Results}

We provide the results of HeRD-Merging for each model to supplement the overall results presented in \Cref{section:merging-comparisons}. HeRD-Merging demonstrates competitive performance with current merging baselines (\Cref{table:weight_merging_adherence}). On Qwen, HeRD-Merging reaches $92.4\%$ mean adherence, exceeding the strongest baseline by $2.2$ percentage points. Similarly, HeRD-Merging reaches $78.9\%$ on Mistral, exceeding the strong baseline, task arithmetic, which scores $76.7\%$. On Gemma, HeRD-Merging matches task arithmetic at $68.5\%$ mean adherence, while DARE-TIES leads at $71.9\%$ and exceeds HeRD-Merging on all three behaviors. This suggests that although performance improvements are displayed on Qwen and Mistral, HeRD-Merging is not a universally superior merging procedure in comparison to existing methods \footnote{The core aim of this work is in improving steering performance, not merging performance. We report these results merely to demonstrate that using HeRD-Merging results in merging performance similar to existing procedures.}. Nevertheless, the results support HeRD-Merging as a competitive weight-merging procedure that does not dramatically sacrifice merging performance for Hessian regularization.

\begin{table*}[h]
\centering
\caption{Behavioral adherence (\%) of model merging methods with 95\% confidence intervals, computed across all combinations of target behaviors and reported individually for each. Bold indicates the best result per column within each model block.}
\footnotesize
\renewcommand{\arraystretch}{1.17}
\setlength{\tabcolsep}{4pt}
\begin{tabular}{@{}llcccc@{}}
\toprule
\textbf{Model} & \textbf{Model Merging Method} & \textbf{Analogy Use} & \textbf{Bulletpoint Use} & \shortstack{\textbf{Sophisticated}\\\textbf{Language}} & \textbf{Average} \\
\midrule
\multirow{6}{*}{\rotatebox[origin=c]{90}{Qwen2.5-VL-3B}}
 & HeRD-Merging (Ours)  & \textbf{86.8 [82.5, 90.5]} & 96.2 [94.2, 98.0] & \textbf{94.2 [91.0, 97.0]} & \textbf{92.4} \\
\cmidrule(l){2-6}
 & DARE task arithmetic & 81.0 [76.2, 85.5] & \textbf{99.0 [97.0, 100.0]} & 90.5 [86.5, 94.0] & 90.2 \\
 & Task arithmetic      & 82.8 [78.2, 87.0] & 98.5 [96.0, 100.0]          & 87.0 [82.0, 91.5] & 89.4 \\
 & DARE-TIES            & 73.5 [67.8, 79.0] & 98.5 [96.0, 100.0]          & 87.8 [83.8, 91.5] & 86.6 \\
 & KnOTS-TIES           & 44.0 [37.2, 50.7] & 98.2 [95.8, 99.8]           & 63.7 [58.2, 69.0] & 68.6 \\
 & TIES                 & 43.0 [36.0, 49.8] & 84.2 [80.0, 88.2]           & 54.2 [49.2, 59.5] & 60.5 \\
\midrule
\multirow{6}{*}{\rotatebox[origin=c]{90}{\shortstack{Mistral-7B \\ Instruct-v0.3}}}
 & HeRD-Merging (Ours)  & 76.0 [71.0, 81.2] & \textbf{93.8 [91.0, 96.2]} & \textbf{66.8 [60.5, 72.5]} & \textbf{78.9} \\
\cmidrule(l){2-6}
 & DARE task arithmetic & \textbf{77.5 [73.0, 82.2]} & 90.0 [86.8, 93.0] & 60.8 [54.5, 66.8] & 76.1 \\
 & Task arithmetic      & \textbf{77.5 [72.8, 82.2]} & 91.2 [88.2, 94.0] & 61.3 [55.0, 67.5] & 76.7 \\
 & DARE-TIES            & 72.2 [68.0, 76.2] & 89.5 [85.8, 92.5]          & 63.0 [57.0, 69.2] & 74.9 \\
 & KnOTS-TIES           & 65.5 [61.0, 69.8] & 77.5 [73.0, 81.8]          & 43.5 [39.0, 48.0] & 62.2 \\
 & TIES                 & 62.5 [57.8, 67.2] & 68.8 [63.0, 74.2]          & 54.2 [48.2, 60.5] & 61.8 \\
\midrule
\multirow{6}{*}{\rotatebox[origin=c]{90}{Gemma-4-E2B}}
 & HeRD-Merging (Ours)  & 61.5 [57.2, 65.8] & 93.2 [90.0, 96.0]          & 50.7 [45.8, 55.8] & 68.5 \\
\cmidrule(l){2-6}
 & DARE task arithmetic & 62.0 [57.8, 66.0] & 91.0 [87.8, 94.0]          & 54.8 [48.5, 61.0] & 69.3 \\
 & Task arithmetic      & 62.3 [58.0, 66.5] & 93.2 [89.5, 96.2]          & 50.0 [45.2, 55.0] & 68.5 \\
 & DARE-TIES            & \textbf{64.2 [60.2, 68.2]} & 95.0 [92.5, 97.2] & 56.5 [50.2, 62.8] & \textbf{71.9} \\
 & KnOTS-TIES           & 44.0 [40.8, 47.2] & \textbf{95.8 [92.5, 98.2]} & 49.5 [42.8, 56.5] & 63.1 \\
 & TIES                 & 48.8 [44.0, 53.2] & 75.2 [69.2, 80.8]          & \textbf{56.8 [50.2, 63.2]} & 60.3 \\
\bottomrule
\end{tabular}
\label{table:weight_merging_adherence}
\end{table*}

\subsection{First-Order Steering at All Layers}
\label{section:all-layer-experiments}

The evaluation of first-order steering against the remaining steering baselines utilized only a single layer for steering to facilitate a fair comparison between all methods. However, first-order steering supports a clean formulation for computing steering vectors for all layers simultaneously without encountering output degradation. The procedure for doing this remains consistent with \Cref{section:steering-results}, with the difference being that the LoRA adapters used are the ones trained at all layers for use in \Cref{section:merging-comparisons}.

HeRD-Merging retains higher mean adherence than task arithmetic when steering at all layers, but the benefit of all-layer steering itself depends on the model (\Cref{table:single_vs_all_layer_steering}). Relative to single-layer steering, mean adherence with HeRD-Merging increases from $64.7\%$ to $81.2\%$ on Qwen and from $32.0\%$ to $34.3\%$ on Gemma, but decreases from $42.5\%$ to $35.4\%$ on Mistral. Thus, applying steering at more layers does not uniformly improve behavioral control.

To characterize how steering-vector contributions are distributed across layers, we define the \emph{effective layer count} for target behavior $i$ as
\begin{equation}
P_i = \frac{1}{\sum_\ell p_{i,\ell}^2},
\qquad
p_{i,\ell} = \frac{\|\mA_\ell^*\ve_i\|_2^2}{\sum_r \|\mA_r^*\ve_i\|_2^2},
\end{equation}
assuming that at least one of the steering vectors is nonzero. This definition has the property that if the summed steering vector norm is distributed equally over $K$ layers and is zero elsewhere, then $P_i = K$. As shown in \Cref{table:steering-effective-layer-count}, Qwen has the lowest effective counts ($5.5$--$8.6$, mean $7.4$), Mistral is intermediate ($9.9$--$12.9$, mean $11.1$), and Gemma has the highest counts ($19.9$--$20.8$, mean $20.3$). Within Qwen, sophisticated-language steering is the most concentrated, whereas Gemma exhibits similarly distributed energy across all three behaviors. We remark that steering at all layers simultaneously is underexplored due to the pre-existing problem of representational collapse \citep{braun2025understanding, tan2024analysing, nelwan2026deployableperinstancemultilayeractivation}. However, we encourage future work to investigate the relationship between the steering vector norms produced from all-layer steering and the steerability of particular models or model layers.

\begin{table*}[t]
\centering
\caption{
Behavioral adherence (\%) of first-order steering at all layers with 95\% confidence intervals, computed across all combinations of target behaviors and reported individually for each. Bold indicates the best result per column within each model block.
}
\footnotesize

\setlength{\tabcolsep}{4pt}
\renewcommand{\arraystretch}{1.25}

\begin{tabular}{@{}llcccc@{}}
\toprule
\textbf{Model} & \textbf{Steering Method / Layers}
& \textbf{Analogy Use}
& \textbf{Bulletpoint Use}
& \shortstack{\textbf{Sophisticated}\\\textbf{Language}}
& \textbf{Average} \\
\midrule
\multirow{4}{*}{\rotatebox[origin=c]{90}{\shortstack{Qwen2.5-VL\\3B}}}
 & HeRD-Merging (All Layers)
 & 66.0 [62.0, 69.8]
 & \textbf{89.8} [86.0, 93.2]
 & \textbf{87.8} [84.2, 90.8]
 & \textbf{81.2} \\
& Task Arithmetic (All Layers)
 & \textbf{69.8} [65.5, 74.0]
 & 89.2 [84.5, 93.3]
 & 70.2 [66.0, 74.8]
 & 76.4 \\
 & HeRD-Merging (Single Layer)
 & 42.2 [36.8, 48.0]
 & 74.2 [68.2, 80.0]
 & 77.8 [73.0, 82.5]
 & 64.7 \\
 & Task Arithmetic (Single Layer)
 & 43.0 [38.2, 47.8]
 & 67.0 [61.0, 73.0]
 & 56.2 [49.5, 63.2]
 & 55.4 \\
\midrule

\multirow{4}{*}{\rotatebox[origin=c]{90}{\shortstack{Mistral-7B\\Instruct-v0.3}}}
 & HeRD-Merging (All Layers)
 & 0.3 [0.0, 0.9]
 & \textbf{64.7} [55.9, 72.8]
 & 41.2 [33.1, 50.3]
 & 35.4 \\
 & Task Arithmetic (All Layers)
 & 0.3 [0.0, 0.9]
 & 64.1 [54.7, 73.1]
 & 33.4 [27.5, 39.7]
 & 32.6 \\
 & HeRD-Merging (Single Layer)
 & 16.6 [15.3, 17.8]
 & 63.5 [60.7, 66.5]
 & \textbf{47.5} [44.3, 50.8]
 & \textbf{42.5} \\
 & Task Arithmetic (Single Layer)
 & \textbf{17.1} [15.8, 18.3]
 & 63.7 [61.0, 66.5]
 & 41.3 [38.7, 44.0]
 & 40.7 \\

\midrule

\multirow{4}{*}{\rotatebox[origin=c]{90}{\shortstack{Gemma-4\\E2B}}}
 & HeRD-Merging (All Layers)
 & 0.3 [0.0, 0.9]
 & 56.9 [47.8, 65.6]
 & \textbf{45.6} [37.8, 54.1]
 & \textbf{34.3} \\
 & Task Arithmetic (All Layers)
 & 0.3 [0.0, 0.9]
 & 55.9 [46.6, 64.7]
 & 45.0 [37.5, 53.1]
 & 33.7 \\
 & HeRD-Merging (Single Layer)
 & \textbf{0.5} [0.1, 1.1]
 & \textbf{60.1} [56.4, 63.8]
 & 35.4 [32.4, 38.7]
 & 32.0 \\
 & Task Arithmetic (Single Layer)
 & \textbf{0.5} [0.1, 1.0]
 & 59.9 [56.4, 63.6]
 & 35.1 [31.9, 38.5]
 & 31.8 \\
\bottomrule
\end{tabular}
\label{table:single_vs_all_layer_steering}
\end{table*}

\begin{table}[t]
\centering
\caption{Effective layer count for each behavioral steering vector.}
\small
\setlength{\tabcolsep}{6pt}
\begin{tabular}{lcccc}
\toprule
\textbf{Model} &
\textbf{Analogy} &
\textbf{Bulletpoint Use} &
\textbf{Sophisticated Language} &
\textbf{Mean} \\
\midrule
Qwen2.5-VL-3B & 8.6 & 8.0 & 5.5 & 7.4 \\
Mistral-7B & 9.9 & 12.9 & 10.6 & 11.1 \\
Gemma-4-E2B & 19.9 & 20.0 & 20.8 & 20.3 \\
\bottomrule
\end{tabular}
\label{table:steering-effective-layer-count}
\end{table}

\begin{figure}[H]
    \centering
    \includegraphics[width=\textwidth]{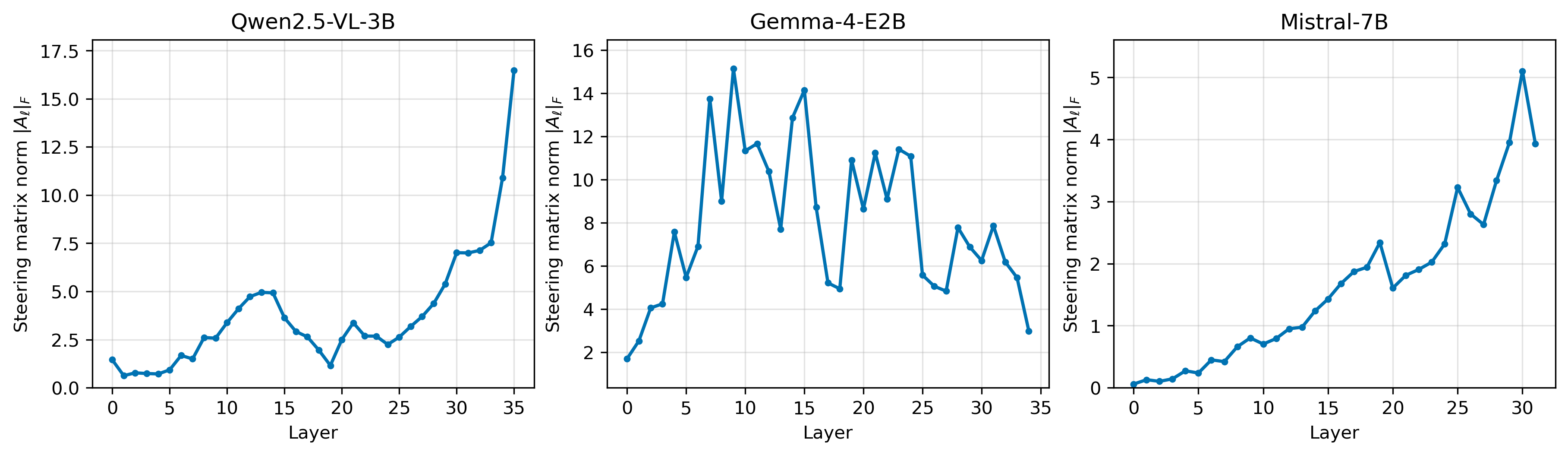}
    \caption{Frobenius norm of the steering matrix at each layer.}
    \label{fig:steering_matrix_frobenius_norm}
\end{figure}

\noindent\begin{minipage}{\linewidth}
\subsection{Synthetic Data Examples}
\label{section:dataset-examples}

\begin{table}[H]
\centering
\footnotesize
\setlength{\tabcolsep}{3pt}
\begin{tabularx}{\textwidth}{@{}>{\raggedright\arraybackslash}p{2.6cm}LL@{}}
\toprule
\multicolumn{3}{@{}l@{}}{\textbf{Target Behavior:} Analogy Use} \\
\midrule
\textbf{Prompt} & \textbf{Original Answer} & \textbf{New Answer} \\
\midrule

Write an analytical essay on metaphorical language in nature-themed writing.
& Nature-themed writing often employs metaphorical language to describe the natural environment in a way that is both imaginative and sensory. Metaphors serve as a tool for writers to create vivid images that allow readers to connect with the natural world beyond the physical senses. In this essay, we will analyze the use of metaphorical language in nature-themed writing and how it contributes to the reader's \ldots{}
& Imagine a vast garden where each plant represents a different element of writing about nature, and metaphorical language is the vine that weaves through this garden, connecting and enhancing every flower and tree. Just as vines wrap around trees, giving them added texture and depth, metaphorical language encircles simple descriptions of nature, transforming them into rich tapestries of meaning. \newline In this garden, \ldots{} \\
\midrule

Please provide a step-by-step guide on how to make a traditional beef stroganoff, including ingredients, cooking methods, and serving suggestions.
& Ingredients: \newline \textbullet{} 1 lb top sirloin beef, sliced thinly \newline \textbullet{} 1 onion, sliced thin \newline \textbullet{} 2 cloves of garlic, minced \newline \textbullet{} 8 oz cremini mushrooms, sliced \newline \textbullet{} 1-2 tbsp all-purpose flour \newline \textbullet{} 1/2 cup beef broth \newline \textbullet{} 1/2 cup sour cream \newline \textbullet{} 2 tbsp butter \newline \textbullet{} Salt and pepper to taste \newline For the noodles: \newline \textbullet{} 8 oz egg noodles \newline \textbullet{} Salt \newline Instructions: \newline 1. Cook the egg noodles according to the package instructions with a generous pinch of salt. Drain and set \ldots{}
& Making traditional beef stroganoff is like preparing a classic orchestra performance, where each musician (ingredient) and section (step) must come in harmony for a beautiful symphony (delicious dish). \newline First, gather your orchestra members: tender beef strips (the solo violinist), mushrooms, onions, garlic (the woodwind section), sour cream, beef broth (the brass section), flour (the percussion), butter, oil, salt, \ldots{} \\
\midrule

Write a Tcl script that generates a random password with a length of 8 characters, comprising of uppercase letters, lowercase letters, and digits. The generated password should be displayed on the console.
& \texttt{Set password ""} \newline \texttt{set possible\_chars "abcdefghijklmnopqrstuvw\allowbreak{}xyzABCDEFGHIJKLMNOPQRSTU\allowbreak{}VWXYZ1234567890"} \newline \texttt{for \{set I 0\} \{\$i \textless{} 8\} \{incr i\} \{} \newline \texttt{\ \ set rand\_index [expr \{int(rand() * [string length \$possible\_chars])\}]} \newline \texttt{\ \ append password [string index \$possible\_chars \$rand\_index]} \newline \texttt{\}} \newline \texttt{puts "Your random password is: \$password"}
& Imagine you're in a kitchen, and you want to make a sandwich that's exactly 8 layers thick. Each layer can be a slice of bread (uppercase letter), a slice of cheese (lowercase letter), or a slice of tomato (digit). To make the sandwich, you randomly grab one type of ingredient from the bowl each time until you have all 8 layers stacked. Once done, you place the sandwich on the counter for everyone to see. \newline So, the \ldots{} \\
\bottomrule

\end{tabularx}

\end{table}
\end{minipage}\par

\begin{table}[H]
\centering
\footnotesize
\setlength{\tabcolsep}{3pt}
\begin{tabularx}{\textwidth}{@{}>{\raggedright\arraybackslash}p{2.6cm}LL@{}}
\toprule
\multicolumn{3}{@{}l@{}}{\textbf{Target Behavior:} Bullet Points} \\
\midrule
\textbf{Prompt} & \textbf{Original Answer} & \textbf{New Answer} \\
\midrule

Write an analytical essay on metaphorical language in nature-themed writing.
& Nature-themed writing often employs metaphorical language to describe the natural environment in a way that is both imaginative and sensory. Metaphors serve as a tool for writers to create vivid images that allow readers to connect with the natural world beyond the physical senses. In this essay, we will analyze the use of metaphorical language in nature-themed writing and how it contributes to the reader's \ldots{}
& \textbullet{} Metaphorical language in nature-themed writing serves as a powerful tool to convey complex emotions and abstract concepts through relatable and vivid imagery.   \newline \textbullet{} It allows writers to bridge the gap between human experience and the natural world, making the unfamiliar familiar and the intangible tangible.   \newline \textbullet{} Nature metaphors often draw on familiar elements like seasons, animals, landscapes, and weather phenomena \ldots{} \\
\midrule

Please provide a step-by-step guide on how to make a traditional beef stroganoff, including ingredients, cooking methods, and serving suggestions.
& Ingredients: \newline \textbullet{} 1 lb top sirloin beef, sliced thinly \newline \textbullet{} 1 onion, sliced thin \newline \textbullet{} 2 cloves of garlic, minced \newline \textbullet{} 8 oz cremini mushrooms, sliced \newline \textbullet{} 1-2 tbsp all-purpose flour \newline \textbullet{} 1/2 cup beef broth \newline \textbullet{} 1/2 cup sour cream \newline \textbullet{} 2 tbsp butter \newline \textbullet{} Salt and pepper to taste \newline For the noodles: \newline \textbullet{} 8 oz egg noodles \newline \textbullet{} Salt \newline Instructions: \newline 1. Cook the egg noodles according to the package instructions with a generous pinch of salt. Drain and set \ldots{}
& \textbullet{} Ingredients: \newline \quad\textbullet{} 1 lb (450g) beef sirloin or tenderloin, thinly sliced into strips \newline \quad\textbullet{} 1 medium onion, finely chopped \newline \quad\textbullet{} 2 cloves garlic, minced \newline \quad\textbullet{} 8 oz (225g) mushrooms, sliced \newline \quad\textbullet{} 1 cup beef broth \newline \quad\textbullet{} 1 cup sour cream \newline \quad\textbullet{} 2 tablespoons all-purpose flour \newline \quad\textbullet{} 2 tablespoons butter \newline \quad\textbullet{} 2 tablespoons vegetable oil \newline \quad\textbullet{} 1 tablespoon Dijon mustard (optional) \newline \quad\textbullet{} Salt and black pepper to taste \newline \quad\textbullet{} Fresh parsley, \ldots{} \\
\midrule

Write a Tcl script that generates a random password with a length of 8 characters, comprising of uppercase letters, lowercase letters, and digits. The generated password should be displayed on the console.
& \texttt{Set password ""} \newline \texttt{set possible\_chars "abcdefghijklmnopqrstuvw\allowbreak{}xyzABCDEFGHIJKLMNOPQRSTU\allowbreak{}VWXYZ1234567890"} \newline \texttt{for \{set I 0\} \{\$i \textless{} 8\} \{incr i\} \{} \newline \texttt{\ \ set rand\_index [expr \{int(rand() * [string length \$possible\_chars])\}]} \newline \texttt{\ \ append password [string index \$possible\_chars \$rand\_index]} \newline \texttt{\}} \newline \texttt{puts "Your random password is: \$password"}
& \textbullet{} \texttt{\#!/usr/bin/env tclsh} \newline \textbullet{} \texttt{\# Define character sets} \newline \textbullet{} \texttt{set uppercase "ABCDEFGHIJKLMNOPQRSTUVW\allowbreak{}XYZ"} \newline \textbullet{} \texttt{set lowercase "abcdefghijklmnopqrstuvw\allowbreak{}xyz"} \newline \textbullet{} \texttt{set digits "0123456789"} \newline \textbullet{} \texttt{\# Combine all characters} \newline \textbullet{} \texttt{set all\_chars "\$\{uppercase\}\$\{lowercase\allowbreak{}\}\$\{digits\}"} \newline \textbullet{} \texttt{\# Set password length} \newline \textbullet{} \texttt{set password\_length 8} \newline \textbullet{} \texttt{\# Initialize empty password} \newline \textbullet{} \texttt{set password ""} \newline \textbullet{} \texttt{\# Seed random number generator with current time} \ldots{} \\
\bottomrule

\end{tabularx}

\end{table}

\begin{table}[H]
\centering
\footnotesize
\setlength{\tabcolsep}{3pt}
\begin{tabularx}{\textwidth}{@{}>{\raggedright\arraybackslash}p{2.6cm}LL@{}}
\toprule
\multicolumn{3}{@{}l@{}}{\textbf{Target Behavior:} Sophisticated Language} \\
\midrule
\textbf{Prompt} & \textbf{Original Answer} & \textbf{New Answer} \\
\midrule

Write an analytical essay on metaphorical language in nature-themed writing.
& Nature-themed writing often employs metaphorical language to describe the natural environment in a way that is both imaginative and sensory. Metaphors serve as a tool for writers to create vivid images that allow readers to connect with the natural world beyond the physical senses. In this essay, we will analyze the use of metaphorical language in nature-themed writing and how it contributes to the reader's \ldots{}
& In the labyrinthine corridors of literary exegesis, the examination of metaphorical language within the ambit of nature-themed writing emerges as an intellectually fecund endeavor, replete with intricate interplays of semiotics and ontological resonances. The metaphor, that quintessentially abstruse rhetorical device, functions not merely as an ornamental linguistic flourish but as an epistemic conduit through which \ldots{} \\
\midrule

Please provide a step-by-step guide on how to make a traditional beef stroganoff, including ingredients, cooking methods, and serving suggestions.
& Ingredients: \newline \textbullet{} 1 lb top sirloin beef, sliced thinly \newline \textbullet{} 1 onion, sliced thin \newline \textbullet{} 2 cloves of garlic, minced \newline \textbullet{} 8 oz cremini mushrooms, sliced \newline \textbullet{} 1-2 tbsp all-purpose flour \newline \textbullet{} 1/2 cup beef broth \newline \textbullet{} 1/2 cup sour cream \newline \textbullet{} 2 tbsp butter \newline \textbullet{} Salt and pepper to taste \newline For the noodles: \newline \textbullet{} 8 oz egg noodles \newline \textbullet{} Salt \newline Instructions: \newline 1. Cook the egg noodles according to the package instructions with a generous pinch of salt. Drain and set \ldots{}
& Ah, the culinary undertaking of preparing a traditional Beef Stroganoff---a gastronomic symphony that demands both precision and reverence for its Russo-European origins. Permit me to elucidate, in a manner befitting the grandeur of this dish, the sequential methodology, inclusive of requisite victuals, intricate culinary manipulations, and the most perspicacious recommendations for its consummate presentation. \ldots{} \\
\midrule

Write a Tcl script that generates a random password with a length of 8 characters, comprising of uppercase letters, lowercase letters, and digits. The generated password should be displayed on the console.
& \texttt{Set password ""} \newline \texttt{set possible\_chars "abcdefghijklmnopqrstuvw\allowbreak{}xyzABCDEFGHIJKLMNOPQRSTU\allowbreak{}VWXYZ1234567890"} \newline \texttt{for \{set I 0\} \{\$i \textless{} 8\} \{incr i\} \{} \newline \texttt{\ \ set rand\_index [expr \{int(rand() * [string length \$possible\_chars])\}]} \newline \texttt{\ \ append password [string index \$possible\_chars \$rand\_index]} \newline \texttt{\}} \newline \texttt{puts "Your random password is: \$password"}
& Indubitably, allow me to proffer unto you a Tcl script of such intricate construction and deliberate elegance that it conjures forth an octet of characters, each meticulously culled from the exalted pantheon of uppercase and lowercase alphabets, as well as the numerals that form the bedrock of arithmetical discourse. The ensuing script shall, with unerring precision, manifest the generated cryptographic artifact \ldots{} \\
\bottomrule

\end{tabularx}

\end{table}

\subsection{Steering Output Examples}

\begin{table}[H]
\centering
\footnotesize
\setlength{\tabcolsep}{3pt}
\begin{tabularx}{\textwidth}{@{}l>{\raggedright\arraybackslash}p{2.7cm}L@{}}
\toprule
\multicolumn{3}{@{}p{\textwidth}@{}}{
\textbf{Target Behavior:} Sophisticated Language.
} \\

\multicolumn{3}{@{}p{\textwidth}@{}}{
\textbf{Question:} What programmes did Channel 5 air during 2002?
} \\
\midrule
\textbf{Model} & \textbf{Method} & \textbf{Output} \\
\midrule

\multirow{7}{*}{\rotatebox[origin=c]{90}{Qwen2.5-VL-3B}}
& First-Order (HeRD-Merging)
& ``During the tumultuous year of 2002, Channel 5, that most illustrious of British broadcasting institutions, aired an array of programmes that ranged from the arcane and esoteric to the mundane and prosaic.'' \\
& First-Order (Task Arithmetic)
& ``During the tumultuous year of 2002, Channel 5, that most illustrious of British broadcasting institutions, aired an array of programmes that were as diverse and eclectic as the cosmos itself. The network's programming diet was a veritable feast of intellectual and cultural savagery \ldots{}'' \\
& Linear-AcT
& ``Channel 5 aired a variety of programs during the year 2002, reflecting the channel's commitment to a broad spectrum of entertainment and informational content. Here is an overview of some of the notable programs \ldots{}'' \\
& CAA
& ``Channel 5, the British television network known for its eclectic programming, aired a diverse array of content during the year 2002. Here is a selection of some notable programmes that were broadcast on the channel \ldots{}'' \\
& RE-Control
& ``I'm sorry, but I cannot provide answers on political matters. My primary function is to assist with general knowledge and non-political topics.'' \\
& RepE
& ``I'm sorry, but I cannot provide answers on political matters. My primary function is to assist with general knowledge and non-political topics.'' \\
& ODESteer
& ``I apologize, but I don't have real-time data or the ability to access current information about specific TV channels' programming schedules.'' \\
\midrule

\multirow{7}{*}{\rotatebox[origin=c]{90}{\shortstack{Mistral-7B \\ Instruct-v0.3}}}
& First-Order (HeRD-Merging)
& ``In the annals of the televisual panorama that was the year 2002, the venerable Channel 5, that venerable and hallowed institution of the televisual firmament, deigned to bestow upon its discerning audience a veritable cornucopia of televis \ldots{}'' \\
& First-Order (Task Arithmetic)
& ``In the annals of the televisual panorama that was the year 2002, the venerable Channel 5, that venerable and hallowed institution of the televisual firmament, deigned to bestow upon its discerning audience a veritable cornucopia of televis \ldots{}'' \\
& ODESteer
& ``It a a a a a a a a a?'' \\
& RepE
& ``The 2020s with an unpardonable fervor \ldots{} the cinematic extravagance of their audacious, labyrinthine narrative narrations \ldots{}'' \\
& Linear-AcT
& ``In the realm of the most prodigious and yet, perchance, to be certain, the forthcoming and forthcoming forthwith, the forthright and prolixity of your inquiry \ldots{}'' \\
& CAA
& ``In the quambleence to Indubmandyly \ldots{}'' \\
& RE-Control
& ``Channel 5, a British television channel, aired various shows during 2002. Here are some of the notable programs that were broadcast that year: Big Brother UK: Series 3 \ldots{}'' \\
\midrule

\multirow{7}{*}{\rotatebox[origin=c]{90}{Gemma-4-E2B}}
& First-Order (HeRD-Merging)
& ``I do not have a specific, comprehensive, and up-to-date television programming log for Channel 5 for the entire year of 2002 in my current knowledge base to \ldots{}'' \\
& First-Order (Task Arithmetic)
& ``I do not have a specific, comprehensive, and up-to-date television programming log for Channel 5 for the entire year of 2002 in my current knowledge base \ldots{}'' \\
& Linear-AcT
& ``I do not Ah Ah Ah Ah Ah Ah Ah \ldots{}'' \\
& CAA
& ``I do not have access to a comprehensive, indexed, year-by-year programme schedule for Channel 5 in 2002. Such in-depth, archival programme data is typically held in proprietary broadcast archives.'' \\
& RepE
& ``I do not have access to a real-time or comprehensive archive of television programming for a specific year like 2002. Therefore, I cannot provide an exhaustive list \ldots{}'' \\
& ODESteer
& ``In arbitrary arbitrary arbitrary \ldots{}'' \\
& RE-Control
& ``I do not have access to a comprehensive, real-time archive of every programme Channel 5 aired during the entire year of 2002.'' \\
\bottomrule
\end{tabularx}

\label{table:qualitative-sophisticated-language}
\end{table}

\begin{table*}[h]
\centering
\footnotesize
\setlength{\tabcolsep}{3pt}
\begin{tabularx}{\textwidth}{@{}l>{\raggedright\arraybackslash}p{2.7cm}L@{}}
\toprule
\multicolumn{3}{@{}p{\textwidth}@{}}{
\textbf{Target Behavior:} Bulletpoint Use

} \\

\multicolumn{3}{@{}p{\textwidth}@{}}{
\textbf{Question:} I think my graphics card is blowing up; a checkerboard pattern keeps showing up on my monitor sometimes.
} \\
\midrule
\textbf{Model} & \textbf{Method} & \textbf{Output} \\
\midrule

\multirow{7}{*}{\rotatebox[origin=c]{90}{Qwen2.5-VL-3B}}
& First-Order (HeRD-Merging)
& ``\textbullet{} \textbf{Check for overheating:} Ensure your computer is not too hot to touch. Check whether fans are spinning normally and cooling effectively. \textbullet{} \textbf{Monitor issues:} Verify whether the monitor is displaying correctly \ldots{}'' \\
& First-Order (Task Arithmetic)
& ``\textbullet{} \textbf{Check the graphics card:} Ensure that the graphics card is properly seated in the PCIe slot. Check for visible damage or signs of overheating. \textbullet{} \textbf{Monitor settings:} Adjust the resolution \ldots{}'' \\
& Linear-AcT
& ``It sounds like your graphics card might be experiencing issues. Here are a few steps you can take to troubleshoot and potentially resolve this issue.'' \\
& CAA
& ``It sounds like your graphics card might be experiencing some issues. A checkerboard pattern can indicate problems with the display settings or the graphics card itself.'' \\
& RE-Control
& ``It sounds like your graphics card may be experiencing some issues. Check whether the graphics card is properly connected to the computer.'' \\
& RepE
& ``It sounds like your graphics card may be experiencing some issues. A checkerboard pattern can indicate a problem with display settings or drivers.'' \\
& ODESteer
& ``It sounds like your graphics card may be experiencing some issues. A checkerboard pattern can be caused by a faulty graphics card, video cable, or monitor.'' \\
\midrule

\multirow{7}{*}{\rotatebox[origin=c]{90}{\shortstack{Mistral-7B \\ Instruct-v0.3}}}
& First-Order (HeRD-Merging)
& ``\textbullet{} Check your graphics card for signs of overheating: \textbullet{} Make sure your graphics card is not overheating by checking for physical damage, thermal throttling, and abnormal temperatures.'' \\
& First-Order (Task Arithmetic)
& ``\textbullet{} Check your graphics card for signs of overheating: \textbullet{} Make sure your graphics card is not overheating by checking for physical damage, thermal throttling, and abnormal temperatures.'' \\
& ODESteer
& ``This I--- I I \ldots{} 1 111 \ldots{}'' \\
& RepE
& ``The issue you're experiencing is a graphical artifact related to the game's graphics. The game's graphics card is not a problem; your graphics card is not a problem.'' \\
& Linear-AcT
& ``\texttt{- \ldots{}\quad - \ldots{}\quad - \ldots{}}'' \\
& CAA
& ``\texttt{- - - - - - - - - - - - - - - - - - \ldots{}}'' \\
& RE-Control
& ``I'm sorry to hear about the issue you're experiencing with your graphics card. The checkerboard pattern is often a symptom of `artifacting.' '' \\
\midrule

\multirow{7}{*}{\rotatebox[origin=c]{90}{Gemma-4-E2B}}
& First-Order (HeRD-Merging)
& ``That sounds like a serious issue, and it's definitely concerning if you suspect your graphics card is failing. The checkerboard pattern is a classic symptom that can point to several problems, ranging from software glitches to actual hardware failure. \textbf{Here is a step-by-step troubleshooting guide} \ldots{}'' \\
& First-Order (Task Arithmetic)
& ``That sounds like a serious issue, and it's definitely concerning if you suspect your graphics card is failing. The checkerboard pattern is a classic symptom that can point to several problems, ranging from software glitches to actual hardware failure. \textbf{Here is a step-by-step troubleshooting guide} \ldots{}'' \\
& Linear-AcT
& ``\texttt{- - - - - - - - - - - - - - - - - - \ldots{}}'' \\
& CAA
& ``\texttt{- - - - - - - - - - - - - - - - - - \ldots{}}'' \\
& RepE
& ``That sounds like a \textbf{serious problem}. \textbf{Immediate Action Required:} Stop using the computer immediately \ldots{}'' \\
& ODESteer
& ``That sounds like a \textbf{serious problem} and a strong indicator that something is wrong with your graphics card. \textbf{Here is a breakdown of what might be happening} \ldots{}'' \\
& RE-Control
& ``That sounds like a serious problem, and it's definitely a sign that something is wrong with your graphics card or potentially the connection.'' \\
\bottomrule
\end{tabularx}

\end{table*}

\begin{table*}[h]
\centering
\footnotesize
\setlength{\tabcolsep}{3pt}
\begin{tabularx}{\textwidth}{@{}l>{\raggedright\arraybackslash}p{2.7cm}L@{}}
\toprule
\multicolumn{3}{@{}p{\textwidth}@{}}{
\textbf{Target Behavior:} Analogy Use
} \\

\multicolumn{3}{@{}p{\textwidth}@{}}{
\textbf{Question:} This page isn't working. almak.co.ke redirected you too many times. Try deleting your cookies.
} \\
\midrule
\textbf{Model} & \textbf{Method} & \textbf{Output} \\
\midrule

\multirow{7}{*}{\rotatebox[origin=c]{90}{Qwen2.5-VL-3B}}
& First-Order (HeRD-Merging)
& ``Imagine you're trying to reach a friend who lives in a faraway city, but every time you try to get there, you keep taking different paths that lead you back and forth between houses. It's like if you were trying to find your way from your house to your friend's house, but each time you take \ldots{}'' \\
& First-Order (Task Arithmetic)
& ``Imagine you're trying to reach a friend who lives in a faraway city, but every time you try to get there, you keep taking different paths that lead you back and forth between houses. It's like if you were trying to find your way from your house to your friend's house, but each time you take \ldots{}'' \\
& Linear-AcT
& ``It seems like the website is having trouble loading because it's trying to redirect you too many times \ldots{} It's like if you were trying to find the right path to load the content you're looking for.'' \\
& CAA
& ``It sounds like you're experiencing a frustrating situation where the website is redirecting you to itself repeatedly, which is causing a `too many redirects' error.'' \\
& RE-Control
& ``It seems like there is an issue with the website you are trying to access, and it's redirecting you too many times. This can happen due to various reasons such as server issues or configuration problems.'' \\
& RepE
& ``It seems like there is an issue with the website you are trying to access, and it's redirecting you too many times. This can happen due to various reasons such as server issues or configuration problems.'' \\
& ODESteer
& ``I apologize, but I don't have real-time data or the ability to access current information about specific websites or their server configurations.'' \\
\midrule

\multirow{7}{*}{\rotatebox[origin=c]{90}{\shortstack{Mistral-7B \\ Instruct-v0.3}}}
& First-Order (HeRD-Merging)
& ``Imagine you're trying to open a door with a key. Your key is like your cookies. When you open a website, it's like you're opening a door with a key. If your cookies are missing, it's like you're trying to open a door with a key but the key \ldots{}'' \\
& First-Order (Task Arithmetic)
& ``Imagine you're trying to open a door with a key. Your key is like your cookies. When you open a website, it's like you're opening a door with a key. If your cookies are missing, it's like you're trying to open a door with a key but the key \ldots{}'' \\
& ODESteer
& ``\texttt{.....-------------------- \{ \{ \{ \{ \{ \ldots{}}'' \\
& RepE
& ``Imagine you're a gardener tending to a vineyard of books. Each book collection is like a garden of books. In this metalecture, imagine the bookshelf as a collection of books \ldots{}'' \\
& Linear-AcT
& ``Imagine your shopping shopping your shopping shopping your shopping \ldots{}'' \\
& CAA
& ``Imagine your shopping shopping your shopping \ldots{} Imagine your choosing your choosing \ldots{}'' \\
& RE-Control
& ``It seems like the website almak.co.ke is experiencing a redirect loop, which is causing the `ERR\_TOO\_MANY\_REDIRECTS' error.'' \\
\midrule

\multirow{7}{*}{\rotatebox[origin=c]{90}{Gemma-4-E2B}}
& First-Order (HeRD-Merging)
& ``This error message, \textbf{`This page isn't working almak.co.ke redirected you too many times. Try deleting your cookies. ERR\_TOO\_MANY\_REDIRECTS'}, is a very common browser error. It means that your browser is stuck in an endless loop \ldots{}'' \\
& First-Order (Task Arithmetic)
& ``This error message, \textbf{`This page isn't working almak.co.ke redirected you too many times. Try deleting your cookies. ERR\_TOO\_MANY\_REDIRECTS'}, is a very common browser error. It means that your browser is stuck in an endless loop \ldots{}'' \\
& Linear-AcT
& ``Imagine Imagine Imagine Imagine Imagine Imagine Imagine \ldots{}'' \\
& CAA
& ``Imagine imagine imagine imagine imagine imagine imagine \ldots{}'' \\
& RepE
& ``The error message \textbf{`almak.co.ke redirected you too many times'} is a very common issue that happens when a website's redirect loop goes wrong. Essentially, the website is trying to send you from Page A to Page B \ldots{}'' \\
& ODESteer
& ``\texttt{\}?\}?\}?\}?\}?\}?\}?\}?\}?\}?\}?\ldots{}}'' \\
& RE-Control
& ``This error message, \textbf{`This page isn't working almak.co.ke redirected you too many times. Try deleting your cookies. ERR\_TOO\_MANY\_REDIRECTS'}, is a very common web browser error.'' \\
\bottomrule
\end{tabularx}
\end{table*}

\stopcontents[appendices]

\end{document}